%% file: example_paper.tex
\documentclass{article}

\usepackage{pgfplots}
\pgfplotsset{compat=1.18}

\usepackage{microtype}
\usepackage{graphicx}
\usepackage{subcaption}
\usepackage{booktabs} 
\usepackage{tikz}
\usetikzlibrary{shapes.geometric, positioning, calc, arrows.meta, backgrounds, fit}
\usepackage{hyperref}
\usepackage{framed} 

\usepackage{tikz}
\usepackage{amsmath}
\usepackage{amssymb}
\usepackage{xcolor}
\usetikzlibrary{arrows.meta, positioning, calc, shapes}
\usepackage[preprint]{icml2026}

\setcitestyle{authoryear,round}

\usepackage{amsmath}
\usepackage{amssymb}
\usepackage{mathtools}
\usepackage{amsthm}

\usepackage[capitalize,noabbrev]{cleveref}

\theoremstyle{plain}
\newtheorem{theorem}{Theorem}[section]
\newtheorem{proposition}[theorem]{Proposition}
\newtheorem{lemma}[theorem]{Lemma}

\theoremstyle{definition}

\theoremstyle{remark}

\usepackage[textsize=tiny]{todonotes}

\icmltitlerunning{Latent Shadows: The Gaussian-Discrete Duality in Masked Diffusion}

\usepackage{tikz}
\usetikzlibrary{shadows, shapes, arrows, positioning, calc} 
\usepackage{amsmath} 

\begin{document}

\twocolumn[
  \icmltitle{Latent Shadows: The Gaussian-Discrete Duality in Masked Diffusion}



  \icmlsetsymbol{equal}{*}

  \begin{icmlauthorlist}
    \icmlauthor{Guinan Chen}{ustc}
    \icmlauthor{Xunpeng Huang}{hkust}
    \icmlauthor{Ying Sun}{hkust}
    \icmlauthor{Shijin Wang}{state}
    \icmlauthor{Yanyong Zhang}{ustc}
    \icmlauthor{Chao Wang}{ustc}
    
  \end{icmlauthorlist}

  \icmlaffiliation{ustc}{University of Science and Technology of China, Hefei, China}
  \icmlaffiliation{hkust}{The Hong Kong University of Science and Technology(Guangzhou), Guangdong, China}
  \icmlaffiliation{state}{State Key Laboratory of Cognitive Intelligence \& iFLYTEK AI Research, Hefei, China}

  \icmlcorrespondingauthor{Chao Wang}{wangchaoai@ustc.edu.cn}

  \icmlkeywords{Machine Learning, ICML}

  \vskip 0.3in
]



\printAffiliationsAndNotice{}  

\begin{abstract}
Masked discrete diffusion is a dominant paradigm for high-quality language modeling where tokens are iteratively corrupted to a mask state, yet its inference efficiency is bottlenecked by the lack of deterministic sampling tools.
While diffusion duality enables deterministic distillation for \emph{uniform} models, these approaches generally underperform masked models and rely on complex integral operators.
Conversely, in the \emph{masked} domain, prior methods typically assume the absence of deterministic trajectories, forcing a reliance on stochastic distillation.
To bridge this gap, we establish explicit Masked Diffusion Duality, proving that the masked process arises as the projection of a continuous Gaussian process via a novel maximum-value index preservation mechanism.
Furthermore, we introduce \emph{Masked Consistency Distillation} (MCD), a principled framework that leverages this duality to analytically construct the deterministic coupled trajectories required for consistency distillation, bypassing numerical ODE solvers.
This result strictly improves upon prior stochastic distillation methods, achieving a 16$\times$ inference speedup without compromising generation quality.
Our findings not only provide a solid theoretical foundation connecting masked and continuous diffusion, but also unlock the full potential of consistency distillation for high-performance discrete generation.
Our code is available at https://anonymous.4open.science/r/MCD-70FD.
\end{abstract}

\input{fig_duality.tex}

\section{Introduction}









Diffusion models~\cite{sohl2015deep,song2019generative,ho2020denoising} have revolutionized generative AI, setting the state of the art in continuous domains including image~\cite{ho2020denoising, rombach2022high}, audio~\cite{kong2020diffwave, liu2023audioldm}, and video synthesis~\cite{ho2022video, wu2023tune, esser2023structure, blattmann2023align}. 
Fundamentally, these models iteratively corrupt complex data distributions into simple priors, then learn to reverse this process for generation. 
This elegant framework has catalyzed a wealth of theoretical advancements and efficient sampling techniques~\cite{song2020score, lu2022dpm,chen2024probability,huang2024reverse}.

However, extending this success to discrete domains like language modeling presents a structural dilemma: the noising and denoising processes in discrete data fundamentally conflict with efficient training and inference in continuous formulations.
On one hand, while continuous diffusion can be applied to text via embedding spaces~\cite{strudel2022self, gao2024empowering}—benefiting from well-established training methodologies and efficient sampling accelerators—the misalignment between the continuous probabilistic latent space and the underlying discrete data distribution often leads to quantization errors and unsatisfactory performance.
On the other hand, discrete diffusion models~\cite{austin2021structured} operate directly on the native categorical tokens~\cite{lou2023discrete, sahoo2024simple, shi2024simplified, ou2024your}. Although this ensures perfect alignment with the data and demonstrates superior stability, it comes at a significant cost: the reliance on stochastic categorical sampling precludes the use of the powerful deterministic sampling techniques and fast distillation methods available for continuous models. Consequently, discrete diffusion models typically suffer from slow inference speeds due to the lack of effective acceleration tools.

To take the best of both worlds, Sahoo et al.~\cite{sahoo2025diffusion} recently introduced the concept of \emph{diffusion duality}, bridging the uniform discrete diffusion and underlying continuous diffusion by a maximum operator, which allows us to adapt consistency distillation~\cite{song2023consistency} to accelerate uniform discrete models.
Given that masked models have emerged as the dominant approach for high-quality language modeling, similar distillation techniques are also adapted to masked diffusion models.
For instance, SDTT~\cite{deschenaux2024beyond} achieves significant speedups but explicitly assumes that a deterministic mapping ``cannot exist,'' forcing a reliance on stochastic approximation.
Concurrent with our work, CD$^4$LM~\cite{liang2026cd4lm} employs a subset masking strategy to couple trajectories but justifies it primarily as a statistical variance reduction technique (Rao-Blackwellization).
However, all of these works leave the theoretical connection between the asymmetric masked process and the continuous world unestablished.
This leads to our core research question:
\begin{framed}
    \centering
    \emph{Can we bring consistency distillation into the masked diffusion with provably diffusion duality?}
\end{framed}
In this work, we answer this question by establishing \emph{Masked Diffusion Duality}. At a high level, our approach differs from the duality found in uniform discrete diffusion, where the $\arg\max$ operator simply maps a continuous random variable to a discrete one. Instead, we construct a mask-noising process where the decision to mask a variable is determined by the preservation of the maximum value's index during the continuous noising process. We prove the existence of a precise schedule alignment such that mapping the continuous gaussian process via this index alignment yields a discrete noise process that shares the same forward transition kernel as masked discrete diffusion, conditioned on the clean data.
Crucially, this analytic construction explicitly establishes the deterministic latent trajectories required for consistency distillation~\cite{song2023consistency}, effectively resolving the complex integral operator issues that hinder prior duality-based methods~\cite{sahoo2025diffusion}.
Moreover, due to the binary state (mask or unmask) for each token in mask discrete diffusion, our constructed discrete evolution of each token can only depend on a single scalar quantity—the margin between the signal and the strongest competing noise.
This observation provides a rigorous foundation for \emph{Masked Consistency Distillation (MCD)}, allowing us to bypass the maintenance of high-dimensional latent noise $\boldsymbol{\epsilon}$ across time steps~\cite{sahoo2025diffusion}, relying instead on a single invariant scalar $u$ to lock the trajectory.
In summary, our contributions are threefold:
\begin{itemize}
    \item We establish the Diffusion Duality for Masked Diffusion by proving it arises as the projection of a continuous Gaussian process, extending duality beyond the uniform case;
    \item We identify Scalar Trajectory Locking, showing that this projection \textbf{simplifies to a scalar thresholding mechanism} that enables efficient coupled trajectory construction.
    \item We propose Masked Consistency Distillation (MCD), a principled framework leveraging this duality that substantially improves generation quality at fixed step counts and \textbf{achieves a 16$\times$ speedup without compromising performance.}
\end{itemize}

\section{Preliminary}
\label{sec:preliminary}

To provide the necessary context for our framework, we review the theoretical foundations of continuous diffusion through the lens of Stochastic Differential Equations (SDEs)~\cite{song2020score}. We then discuss discrete masked diffusion models and identify the structural barriers that prevent the direct application of acceleration techniques.

\subsection{Continuous Diffusion}

Instead of discrete noise injection, continuous diffusion models formulate the forward process as a solution to a linear Stochastic Differential Equation (SDE).

Forward SDE and Gaussian Marginals. Let $x_0 \sim p_{data}(x)$ be the data distribution. The forward process $\{w_t\}_{t \in [0,1]}$ diffuses data into noise according to the Itô SDE:
\begin{equation}
    dw_t = f(t)w_t dt + g(t) d\mathbf{w}_t,
    \label{eq:forward_sde}
\end{equation}
where $\mathbf{w}_t$ is standard Brownian motion. Crucially, for the standard Variance Preserving (VP) formulation, the solution to this SDE yields a closed-form Gaussian marginal distribution at any time $t$:
\begin{equation}
    q(w_t|x_0) = \mathcal{N}(w_t; \tilde{\alpha}_t x_0, \tilde{\sigma}_t^2 I).
    \label{eq:gaussian_marginal}
\end{equation}
Here, $w_t$ can be explicitly sampled as $w_t = \tilde{\alpha}_t x_0 + \tilde{\sigma}_t \epsilon$ with $\epsilon \sim \mathcal{N}(0, I)$. This dual view confirms that the continuous SDE formulation is mathematically equivalent to the noise perturbation process in classic Gaussian diffusion models.

Reverse SDE and Probability Flow ODE. Song et al.\cite{song2020score} demonstrated that the generative process can be modeled as a reverse-time SDE. Furthermore, for any such reverse process, there exists a deterministic Probability Flow ODE (PF-ODE) whose trajectories share the same marginal probability densities $\{p_t(w_t)\}_{t \in [0,1]}$ as the stochastic process:
\begin{equation}
    \frac{dw_t}{dt} = f(t)w_t - \frac{1}{2}g^2(t)\nabla_{w_t} \log p_t(w_t).
    \label{eq:pf_ode}
\end{equation}
The PF-ODE defines a unique, smooth, and deterministic trajectory connecting every data point $x_0$ to a specific noise vector $w_1$.

\textbf{Consistency Distillation (CD).}
Following Song et al.~\cite{song2023consistency}, the deterministic nature of the PF-ODE is the fundamental prerequisite for CD. CD aims to distill the multi-step sampling process into a student model $f_\theta(w_t, t)$ that maps any point $w_t$ on the trajectory directly to its origin $x_0$. The distillation process employs a teacher model $f_{\theta^-}$, typically maintained as the Exponential Moving Average (EMA) of the student parameters. Specifically, given a noisy sample $w_t$ drawn from the forward process, a less noisy sample $w_s$ (where $s < t$) is obtained by numerically solving one PF-ODE step using the teacher $f_{\theta^-}$. The student model is then trained to match the teacher's estimate of the clean sample by minimizing the following consistency loss:
\begin{equation}
    \mathcal{L}_{CD}(\theta, \theta^-) = \mathbb{E}_{t, w_t} \left[ \lambda(t) d\left(f_\theta(w_t, t), f_{\theta^-}(w_s, s)\right) \right],
    \label{eq:cd_loss}
\end{equation}
where $d(\cdot, \cdot)$ is a distance metric and $\lambda(t)$ is a weighting function. This objective enforces the self-consistency property $f_\theta(w_t, t) = f_\theta(w_s, s)$, enabling few-step generation.

\subsection{Discrete Masked Diffusion and the Trajectory Gap}

Discrete Diffusion Models operate directly on categorical data $z_t \in \{1, \dots, K\}^L$. In Masked Diffusion Models, the corruption process utilizes an absorbing state, the mask token \texttt{[M]}.

Forward Masking Process. Tokens from the clean sequence $x_0$ are independently replaced by \texttt{[M]} according to a scalar schedule $\gamma_t \in [0,1]$. Unlike the Gaussian noise injection, this is a discrete transition where a token either retains its value or collapses to the absorbing state:
\begin{equation}
    q(z_t^{(i)}|x_0^{(i)}) = \gamma_t \delta(z_t^{(i)} - x_0^{(i)}) + (1-\gamma_t) \delta(z_t^{(i)} - \text{\texttt{[M]}}).
    \label{eq:masked_marginal}
\end{equation}
Once a token becomes \texttt{[M]}, it remains \texttt{[M]} for all subsequent noise levels (absorbing property).

Reverse Denoising Process. The generative process reverses the masking corruption. It is modeled as a Markov chain $p_\theta(z_s|z_t)$ that approximates the true posterior $q(z_s|z_t, x_0)$. Standard MDMs parameterize this by predicting the clean token logits $\hat{x}_\theta(z_t)$ and marginalizing over the posterior:
\begin{equation}
    p_\theta(z_s|z_t) = q(z_s|z_t, x_0=\hat{x}_\theta(z_t)).
    \label{eq:reverse_discrete}
\end{equation}
Specifically, for a masked token $z_t = \texttt{[M]}$, the model samples $z_s$ based on the predicted probability of the token being unmasked at step $s$, weighted by the network's prediction $\hat{x}_\theta(z_t)$. Visible tokens remain unchanged.

\textbf{The Structural Gap.} The fundamental bottleneck for acceleration is the absence of a deterministic flow in the discrete domain. In continuous CD, the training target $w_s$ is generated by numerically solving the PF-ODE from $w_t$. However, standard MDMs are inherently stochastic jump processes; there is no discrete analogue to the PF-ODE that allows for the deterministic calculation of an intermediate state $z_s$ from $z_t$. Consequently, valid student-teacher trajectory pairs $(z_t, z_s)$ cannot be constructed to minimize a consistency objective (Eq.~\ref{eq:cd_loss}). This highlights a fundamental structural gap between continuous diffusion and discrete generative models. Bridging this gap—by extending the duality observed in uniform diffusion to the asymmetric, absorbing setting and recovering a deterministic latent trajectory—forms the central objective of the next section.


\section{Method}

In this section, we address the fundamental structural gap between continuous diffusion and discrete masked models to enable efficient few-step generation.
The core challenge lies in the absence of a differentiable probability flow ODE for discrete data, which precludes the direct application of consistency distillation.
To resolve this, we first establish the \emph{Masked Diffusion Duality} (Sec.~\ref{sec:duality}), proving that the discrete masked process is not an independent stochastic jump process, but rather the \emph{deterministic projection} of a latent continuous Gaussian diffusion.
Crucially, this latent structure bridges the algorithmic gap: just as continuous ODE solvers (e.g., DDIM solver)~\cite{song2020denoising} function by \emph{preserving the fixed Gaussian noise $\epsilon$} across time, our duality implies that the discrete trajectory is inherently governed by the same invariant latent noise.

We further demonstrate that this high-dimensional constraint simplifies via \emph{Scalar Trajectory Locking} (Sec.~\ref{sec:locking}).
Our analysis reveals that under the masking projection, the complex condition of preserving the full latent vector $\epsilon$ reduces to a much simpler mechanism: the trajectory can be indexed by a fixed scalar threshold $u$.
Therefore, by simply fixing $u$ across time steps, we \emph{analytically replicate} the effect of a perfect ODE solver in the discrete space.

Guided by this insight, we propose \emph{Masked Consistency Distillation (MCD)} (Sec.~\ref{sec:mcd}).
Instead of employing numerical solvers, MCD leverages this locked scalar $u$ to directly generate Coupled Trajectory Pairs (Sec.~\ref{sec:locking}) that are mathematically guaranteed to lie on the same underlying probability flow, enabling efficient distillation via a \emph{Hybrid Consistency Objective}.

\subsection{The Masked Diffusion Duality}
\label{sec:duality}

To bridge the gap between continuous and discrete domains, we first construct a continuous latent proxy for the discrete masked process and establish their strict equivalence in terms of both marginal distributions (static) and deterministic trajectories (dynamic).

\paragraph{Latent Space Construction.}
Let $K$ be the dimension of the extended state space (vocabulary size + 1). Let $k \in \{1, \dots, K-1\}$ denote the index of the ground truth token, and let $\mathbf{x}_0 \in \{0, 1\}^K$ be its corresponding one-hot vector representation (where $\mathbf{x}_0^{(k)}=1$). We implicitly define the mask state \texttt{[M]} as the one-hot vector $\mathbf{e}_K \in \{0, 1\}^K$ where the last entry is 1 (i.e., the mask dimension).

We define the continuous latent state $\mathbf{w}_t \in \mathbb{R}^K$ by applying Gaussian noise to the clean vector $\mathbf{x}_0$ under a variance preservation (VP) schedule ($\tilde{\alpha}_t^2 + \tilde{\sigma}_t^2 = 1$):
\begin{equation}
\mathbf{w}_t = \tilde{\alpha}_t \mathbf{x}_0 + \tilde{\sigma}_t \boldsymbol{\epsilon},
\end{equation}
where $\boldsymbol{\epsilon} \sim \mathcal{N}(\mathbf{0}, \mathbf{I})$ is the noise vector. The discrete observation $\mathbf{z}_t \in \{0, 1\}^K$ is obtained via a \emph{Projection Operator} $\mathcal{P}$, which maps the latent vector to a discrete one-hot state. Specifically, the state remains the ground truth vector $\mathbf{x}_0$ if and only if the signal dimension dominates all other dimensions; otherwise, it collapses to the mask vector $\mathbf{e}_K$: 
\begin{equation}
\label{eq:projection}
\mathbf{z}_t = \mathcal{P}(\mathbf{w}_t) \triangleq \begin{cases}
    \mathbf{x}_0 & \text{if } \mathbf{w}_t^{(k)} > \max_{j \neq k} \mathbf{w}_t^{(j)}, \\ 
    \mathbf{e}_K & \text{otherwise}.
\end{cases}
\end{equation}

\paragraph{Static Duality: Distributional Equivalence.}
First, we ensure this latent construction matches the marginals of the target Masked Diffusion Language Model (MDLM). Let $\gamma_t = P(\mathbf{z}_t = \mathbf{x}_0 \mid \mathbf{x}_0)$ be the discrete signal schedule.

\begin{lemma}[SNR Calibration]
\label{lemma:3.1}
The latent signal-to-noise ratio (SNR) is calibrated to strictly match the discrete unmasking schedule $\gamma_t$.
\end{lemma}

\begin{proof}
From the definition of $\mathcal{P}$ (Eq.~\ref{eq:projection}), the ground truth is preserved ($\mathbf{z}_t = \mathbf{x}_0$) if and only if $\frac{\tilde{\alpha}_t}{\tilde{\sigma}_t} > \max_{j \neq k} \epsilon_j - \epsilon_k$. 
Let $Y$ denote this noise difference term and $F_Y$ be its CDF. We have:
\[
P(\mathbf{z}_t = \mathbf{x}_0) = P\left(Y < \frac{\tilde{\alpha}_t}{\tilde{\sigma}_t}\right) = F_Y\left(\frac{\tilde{\alpha}_t}{\tilde{\sigma}_t}\right).
\]
Since $Y$ is a continuous random variable with support on $\mathbb{R}$, $F_Y$ is strictly increasing and bijective onto $(0, 1)$.
Therefore, for any $\gamma_t \in (0, 1)$, setting $\frac{\tilde{\alpha}_t}{\tilde{\sigma}_t} = F_Y^{-1}(\gamma_t)$ ensures the continuous marginal probability exactly equals $\gamma_t$.
\end{proof}

\emph{Discussion.} This lemma underpins our framework (visualized in Figure~\ref{fig:duality_eq5}) by showing that the projected process exactly simulates the target discrete diffusion. With proper SNR calibration, changes in the latent space consistently map to corresponding changes in the discrete space.

\paragraph{Dynamic Duality: Latent Trajectories.}
To enable Consistency Distillation, we need a deterministic path connecting states across time. In the latent continuous space, this is naturally provided by the DDIM trajectory. Assuming an optimal denoiser, the latent state $\mathbf{w}_t$ at any time $t$ is simply a linear interpolation between the clean data $\mathbf{x}_0$ and the fixed~noise~$\boldsymbol{\epsilon}$:
\begin{equation}
\mathbf{w}_t = \tilde{\alpha}_t \mathbf{x}_0 + \tilde{\sigma}_t \boldsymbol{\epsilon}.
\end{equation}
This equation implies that the entire continuous trajectory is uniquely determined by the pair $(\mathbf{x}_0, \boldsymbol{\epsilon})$. We now show that this deterministic property naturally extends to the discrete~domain.

\begin{proposition}[Deterministic Discrete Trajectories]
\label{proposition:trajectory}
The sequence of discrete observations $\{ \mathbf{z}_t \}_{t \in [0,1]}$ forms a deterministic trajectory uniquely defined by the ground truth $\mathbf{x}_0$ and the noise $\boldsymbol{\epsilon}$.
\end{proposition}

\begin{proof}
Recall that the discrete state is obtained via the projection $\mathbf{z}_t = \mathcal{P}(\mathbf{w}_t)$. Substituting the deterministic DDIM path $\mathbf{w}_t = \tilde{\alpha}_t \mathbf{x}_0 + \tilde{\sigma}_t \boldsymbol{\epsilon}$ into the projection condition (Eq.~\ref{eq:projection}), the state $\mathbf{z}_t$ becomes:
\begin{equation}
\mathbf{z}_t = \begin{cases}
    \mathbf{x}_0 & \text{if } \tilde{\alpha}_t + \tilde{\sigma}_t \epsilon_k > \max_{j \neq k} (\tilde{\sigma}_t \epsilon_j), \\
    \mathbf{e}_K & \text{otherwise}.
\end{cases}
\end{equation}
Here, the condition depends \emph{only} on the time-dependent coefficients $(\tilde{\alpha}_t, \tilde{\sigma}_t)$ and the fixed parameters $(\mathbf{x}_0, \boldsymbol{\epsilon})$. No new randomness is introduced during the process.
Therefore, for any two time steps $s < t$, the states $\mathbf{z}_s$ and $\mathbf{z}_t$ are essentially two different "views" derived from the exact same underlying noise $\boldsymbol{\epsilon}$. This shared dependency guarantees that they lie on the same consistent trajectory, enabling direct student-teacher supervision.
\end{proof}

\emph{Discussion.} This proposition addresses a key limitation of prior masked diffusion methods that rely on stochastic jumps. By showing that corresponding pairs $(z_t, z_s)$ lies on the same deterministic trajectory, we justify consistency distillation: the teacher and student represent the same process at different times, yielding valid regression targets.

\begin{figure*}[t]

\centering
\begin{tikzpicture}[
    font=\sffamily,
    >=LaTeX,
    token/.style={
        draw=gray!30, 
        line width=0.8pt,
        minimum size=0.85cm, 
        rounded corners=1pt, 
        font=\bfseries,
        anchor=center
    },
    visible/.style={
        token,
        fill=blue!10, 
        draw=blue!40, 
        text=blue!60!black
    },
    masked/.style={
        token,
        fill=gray!10, 
        draw=gray!30, 
        text=gray!50
    },
    noisebar/.style={
        fill=teal!20, 
        draw=teal!40,
        line width=0.5pt
    },
    labelnode/.style={
        align=right, 
        font=\small\bfseries, 
        text=black!80
    },
    lossbox/.style={
        draw, thick, dashed, rounded corners, inner sep=3pt
    }
]

\def\noisevalues{3, 9, 2, 8, 5, 1, 7, 4, 6, 2} 
\def\gammat{3.5} 
\def\gammas{6.5} 

\begin{scope}[yshift=1.9cm]
    \node[labelnode, anchor=east] at (-0.8, 0.4) {Shared Latent\\Noise $u$};
    
    \foreach \val [count=\i] in \noisevalues {
        \draw[noisebar] (\i-0.4, 0) rectangle (\i+0.4, \val/6);
    }
    
    \draw[dashed, red!60, line width=1pt] (0.5, \gammat/6) -- (10.5, \gammat/6);
    \node[text=red!60, font=\footnotesize, right] at (10.5, \gammat/6) {$\gamma_t$};
    
    \draw[dashed, green!50!black, line width=1pt] (0.5, \gammas/6) -- (10.5, \gammas/6);
    \node[text=green!50!black, font=\footnotesize, right] at (10.5, \gammas/6) {$\gamma_s$};
\end{scope}

\begin{scope}[yshift=1.0cm]
    \node[labelnode, anchor=east] at (-0.8, 0) {Teacher State\\$z_s$};
    \foreach \val [count=\i] in \noisevalues {
        \pgfmathparse{\val > \gammas ? 1 : 0}
        \ifnum\pgfmathresult=1
            \node[masked] (t\i) at (\i, 0) {\scalebox{0.8}{[M]}};
        \else
            \node[visible] (t\i) at (\i, 0) {$x_0$};
        \fi
    }
\end{scope}

\begin{scope}[yshift=0cm]
    \node[labelnode, anchor=east] at (-0.8, 0) {Student State\\$z_t$};
    \foreach \val [count=\i] in \noisevalues {
        \pgfmathparse{\val > \gammat ? 1 : 0}
        \ifnum\pgfmathresult=1
            \node[masked] (s\i) at (\i, 0) {\scalebox{0.8}{[M]}};
        \else
            \node[visible] (s\i) at (\i, 0) {$x_0$};
        \fi
    }
\end{scope}


\node[lossbox, draw=gray!40, dotted, fit=(s1) (t1)] (box_none) {};
\draw[->, gray!60, thick, dashed] (box_none.south) -- ++(0, -0.45) 
    node[below, text=gray, align=center, font=\scriptsize] 
    {No Loss\\(Identity)};

\node[lossbox, draw=orange!80, fit=(s4) (t4)] (box_kl) {};
\draw[->, orange!80, thick] (box_kl.south) -- ++(0, -0.45) 
    node[below, text=black, align=center, font=\scriptsize] 
    {\textbf{Distillation}\\$\mathcal{L}_{\text{KL}}(p_\theta \| p_{\theta^-})$};

\node[lossbox, draw=green!60!black, fit=(s9) (t9)] (box_ce) {};
\draw[->, green!60!black, thick] (box_ce.south) -- ++(0, -0.45) 
    node[below, text=black, align=center, font=\scriptsize] 
    {\textbf{Reconstruction}\\$\mathcal{L}_{\text{CE}}(x_0, p_\theta)$};

\end{tikzpicture}
\vspace{-0.5em}
\caption{\textbf{Visualization of Coupled Trajectory Construction and Hybrid Objective.} 
By sharing latent noise $u$, we enable three distinct supervision regimes: 
(1) \textbf{Identity} (Gray, left) where both see the input, requiring no gradient update;
(2) \textbf{Distillation} (Orange, middle) where both views are masked ($u > \gamma_s$), ensuring consistency on uncertain regions; 
(3) \textbf{Reconstruction} (Green, right) where the teacher reveals tokens masked for the student ($\gamma_t < u \le \gamma_s$), providing hard ground-truth supervision.}
\label{fig:mcd_framework}
\end{figure*}
\subsection{Scalar Trajectory Locking}
\label{sec:locking}

While Proposition~\ref{proposition:trajectory} guarantees that a deterministic path exists, it relies on the high-dimensional noise vector $\boldsymbol{\epsilon}$. Utilizing such a high-dimensional index for distillation is computationally redundant. We now identify a critical simplification intrinsic to the masking process.

\begin{theorem}[Scalar Trajectory Locking]
Along a deterministic trajectory characterized by a fixed pair $(\mathbf{x}_{0}, \boldsymbol{\epsilon})$, the discrete state transition is mathematically equivalent to a thresholding operation on a fixed scalar uniform variable~$u$.
\end{theorem}

\begin{proof}
Along the trajectory, both $\mathbf{x}_{0}$ and $\boldsymbol{\epsilon}$ are constant. Consequently, the noise difference term defined in Lemma~\ref{lemma:3.1}, $Y \triangleq \max_{j \neq k} \epsilon_j - \epsilon_{k}$, is also constant. 
We apply the Probability Integral Transform to define a scalar $u = F_Y(Y)$, which follows a uniform distribution $\mathcal{U}[0, 1]$. 

Recall the unmasking condition from Eq.~\ref{eq:projection}: $\frac{\tilde{\alpha}_t}{\tilde{\sigma}_t} > Y$. Applying the strictly increasing CDF $F_Y$ to both sides, we~get 
\begin{equation}
F_Y\left(\frac{\tilde{\alpha}_t}{\tilde{\sigma}_t}\right) > F_Y(Y).
\end{equation}
By substituting the calibration from Lemma~\ref{lemma:3.1} (where $F_Y(\frac{\tilde{\alpha}_t}{\tilde{\sigma}_t}) = \gamma_t$) and the definition of $u$, this simplifies to:
\begin{equation}
\gamma_t > u.
\end{equation}
This proves that fixing the scalar $u$ fully defines the deterministic path: tokens are unmasked exactly when the signal schedule $\gamma_t$ exceeds the fixed threshold $u$.
\end{proof}

\emph{Discussion.} This theorem is the key to our efficiency. It reveals that the complex condition of preserving the full latent vector $\boldsymbol{\epsilon}$ reduces to a simple scalar comparison. Therefore, we do not need to maintain the high-dimensional noise; a single scalar $u$ is sufficient to uniquely determine and lock the entire trajectory in discrete masked diffusion models.

\paragraph{Operationalizing Coupled Trajectories.}
This theoretical result translates directly into a scalable algorithm for constructing valid student-teacher pairs, as illustrated in Figure~\ref{fig:mcd_framework}. We generalize the single-token result to a sequence $\mathbf{x}_0 = (\mathbf{x}_0^{(1)}, \dots, \mathbf{x}_0^{(L)})$. Instead of sampling high-dimensional Gaussian noise, we simply sample a single shared noise vector $u \sim \mathcal{U}[0, 1]^L$ to lock the trajectory for the entire sequence.

We define the deterministic masking operator $\textsc{Mask}(\cdot)$ element-wise. For any time step $t$ with signal schedule $\gamma_t$, the discrete state $\mathbf{z}_t^{(i)}$ is determined by:
\begin{equation}
\label{eq:mask_op}
\mathbf{z}_t^{(i)} = \textsc{Mask}(\mathbf{x}_0^{(i)}, u^{(i)}, \gamma_t) \triangleq
\begin{cases}
\mathbf{e}_K, & \text{if } u^{(i)} > \gamma_t,  \\
\mathbf{x}_0^{(i)}, & \text{otherwise}.
\end{cases}
\end{equation}
Crucially, Eq.~\ref{eq:mask_op} is mathematically equivalent to the high-dimensional projection (Eq.~\ref{eq:projection}) but operates in a strictly lower-dimensional space. By reducing the noise source from a $K$-dimensional vector to a single scalar for each token without altering the generative dynamics, we provide a tractable and efficient way to construct the coupled trajectories required for consistency distillation.

\subsection{Masked Consistency Distillation (MCD)}
\label{sec:mcd}

With the deterministic coupling mechanism established in Sec.~\ref{sec:locking}, the distillation process becomes algorithmically straightforward (summarized in Algorithm~\ref{alg:mcd}). We calculate the consistency loss solely based on the analytically constructed pairs $(\mathbf{z}_t, \mathbf{z}_s)$, eliminating the need for expensive ODE solvers.

\paragraph{Hybrid Consistency Objective.}
Since the teacher's prediction $p_{\theta^-}(\mathbf{x}_0|\mathbf{z}_s)$ from a noisy state may exhibit high entropy, we apply temperature scaling ($\tau < 1$) to sharpen the target distribution. The student $p_\theta(\mathbf{x}_0|\mathbf{z}_t)$ is then trained to match this target via a hybrid objective.

Let $m_t^{(i)} = \mathbb{I}(u^{(i)} > \gamma_t)$ be the binary mask indicator derived from the shared noise $u$. The total objective combines distillation ($D_{\text{KL}}$) on mutually masked regions with reconstruction ($\mathcal{L}_{\text{CE}}$) on tokens visible only to the teacher:
\begin{equation}
\label{eq:total_loss}
\begin{aligned}
\mathcal{L}_{\text{MCD}} = \sum_{i=1}^L \Big[ & \underbrace{m_s^{(i)} \cdot D_{\text{KL}}(p_\theta(\cdot|\mathbf{z}_t) \parallel p_{\theta^-}(\cdot|\mathbf{z}_s; \tau))}_{\text{Distillation}} \\
+ & \underbrace{(m_t^{(i)} - m_s^{(i)}) \cdot \mathcal{L}_{\text{CE}}(\mathbf{x}_0^{(i)}, p_\theta(\mathbf{x}_0^{(i)}|\mathbf{z}_t))}_{\text{Reconstruction}} \Big].
\end{aligned}
\end{equation}
The rationale behind this hybrid design hinges on the information asymmetry between the two views.
The core of the distillation occurs in the regions masked for both models ($m_s^{(i)}=1$). Since the teacher observes a strictly larger set of visible tokens (specifically those where $m_t^{(i)} - m_s^{(i)} = 1$), its contextual understanding is superior, yielding more accurate and lower-entropy predictions even for the remaining masked tokens. Minimizing $D_{\text{KL}}$ effectively transfers this enhanced contextual reasoning from the teacher to the student.
Meanwhile, the difference term $(m_t^{(i)} - m_s^{(i)})$ serves as a regularization signal. For these tokens, which are visible only to the teacher, we use the ground truth $\mathbf{x}_0$ to provide a stable, zero-variance target. This prevents the student's predictions from drifting, without strictly penalizing it for failing to hallucinate information that is theoretically inaccessible at time $t$. Finally, for tokens visible to both, the loss is naturally zero as no generation is required.


\section{Experiments}
\label{sec:experiments}

We evaluate the effectiveness of Masked Consistency Distillation (MCD). Grounded in the Masked Diffusion Duality framework introduced in Section~\ref{sec:duality}, our approach extends consistency distillation to the discrete domain. We focus on two key questions: (1) Can MCD effectively distill the high-step masked diffusion process into few-step generators while maintaining sample quality? (2) How does MCD compare against leading discrete distillation baselines?

\begin{algorithm}[tb]
\caption{Masked Consistency Distillation (MCD)}
\label{alg:mcd}
\begin{algorithmic}
\STATE \textbf{Input:} Pretrained parameters $\theta_{\text{init}}$, Dataset $\mathcal{D}$, learning rate $\eta$, signal schedule $\gamma(\cdot)$, initial step size $\delta_0$, stages $N$, iterations per stage $M$.

\STATE \textbf{Initialize:} Student $\theta \leftarrow \theta_{\text{init}}$, Teacher $\theta^- \leftarrow \theta$, $\delta \leftarrow \delta_0$

\FOR{$i = 1$ {\bfseries to} $N$} 
    \STATE \textit{// Hard reset: Teacher snapshots Student}
    \STATE $\theta^- \leftarrow \text{stopgrad}(\theta)$ 
    
    \FOR{$j = 1$ {\bfseries to} $M$}
        \STATE Sample data $\mathbf{x}_0 \sim \mathcal{D}$
        
        \STATE Sample shared trajectory noise $u \sim \mathcal{U}[0,1]^L$
        
        \STATE Sample $t \sim \mathcal{U}(\delta, 1]$, set $s \leftarrow t - \delta$
        
        \STATE \textbf{// Coupled trajectory construction}
        \STATE $\mathbf{z}_t \leftarrow \textsc{Mask}(\mathbf{x}_0, u, \gamma_t)$
        \STATE $\mathbf{z}_s \leftarrow \textsc{Mask}(\mathbf{x}_0, u, \gamma_s)$
        
        \STATE \textbf{// Masked Consistency distillation}
        \STATE $\mathbf{p}_{\text{tea}} \leftarrow p_{\theta^-}(\cdot \mid \mathbf{z}_s)$ 
        \STATE $\mathbf{p}_{\text{stu}} \leftarrow p_{\theta}(\cdot \mid \mathbf{z}_t)$
        
        \STATE $\mathcal{L}_{\text{MCD}} \leftarrow 
        \textsc{Loss}(\mathbf{p}_{\text{stu}}, \mathbf{p}_{\text{tea}}; \mathbf{x}_0, \gamma_t, \gamma_s)$
        
        \STATE $\theta \leftarrow \theta - \eta \nabla_\theta \mathcal{L}_{\text{MCD}}$
        \COMMENT{Increase time gap}
    \ENDFOR
    
    \STATE $\delta \leftarrow 2 \cdot \delta$ 
\ENDFOR

\STATE \textbf{return} $\theta$
\end{algorithmic}
\end{algorithm}

\subsection{Experimental Setup}

Following prior work~\cite{sahoo2024simple, deschenaux2024beyond, sahoo2025diffusion}, we conduct experiments on the OpenWebText dataset and employ a Small transformer backbone (169M parameters, $L=12$, $D=768$) based on the Diffusion Transformer (DiT) architecture. 
For the distillation teacher, we utilize a pre-trained MDLM~\cite{sahoo2024simple} checkpoint trained for 1M steps, which serves as the common starting point for both our method and the SDTT baseline.
We primarily compare our method against SDTT~\cite{deschenaux2024beyond}, the current state-of-the-art distillation method for masked diffusion models, using an identical model architecture and evaluation protocol.
Additionally, we provide comparisons with Duo~\cite{sahoo2025diffusion} to benchmark our masked approach against the uniform diffusion domain.

\textbf{Training Details.} We perform distillation for a total of 50,000 steps on 2 NVIDIA A100 GPUs with a global batch size of 128. We use the AdamW optimizer with a learning rate of $6 \times 10^{-5}$ and a 500-step linear warmup. 
For the sharpening temperature $\tau$, we initialize it at $0.96$ in the first round and linearly decrease it by $0.03$ for each subsequent round.
Training is performed in \texttt{bfloat16} precision. To ensure accurate evaluation, we strictly follow the protocol from~\cite{zheng2024masked} and perform all sampling steps in \texttt{float64} precision to avoid rounding errors common in discrete diffusion.

\textbf{Distillation Schedule.} We adopt the staged distillation curriculum from the Discrete Consistency Distillation framework~\cite{sahoo2025diffusion}. The training consists of 5 rounds, with each round lasting 10,000 steps. Following this protocol, we double the time gap $\delta$ between the student and teacher trajectories at the beginning of each round (starting from $\delta=1/512$) and perform a hard update of the teacher parameters at the end of each stage. Crucially, this schedule matches the total training budget of the SDTT baseline~\cite{deschenaux2024beyond}, ensuring a fair comparison.

\definecolor{tea_blue}{RGB}{70, 110, 160}
\definecolor{mcd_r1}{RGB}{255, 215, 165}
\definecolor{mcd_r2}{RGB}{255, 204, 128}
\definecolor{mcd_r3}{RGB}{255, 183, 77}
\definecolor{mcd_r4}{RGB}{245, 124, 0}
\definecolor{mcd_r5}{RGB}{230, 81, 0}

\pgfplotscreateplotcyclelist{mcd_sdtt_final}{
    {tea_blue, mark=*, mark size=1.8pt, thick},    
    {mcd_r1, mark=star, mark size=2.2pt, thick},    
    {mcd_r2, mark=diamond*, mark size=1.8pt, thick},
    {mcd_r3, mark=triangle*, mark size=1.8pt, thick},
    {mcd_r4, mark=square*, mark size=1.6pt, thick}, 
    {mcd_r5, mark=pentagon*, mark size=1.8pt, thick}
}

\begin{figure}[t]
    \centering
    \begin{tikzpicture}
        \begin{semilogxaxis}[
            width=0.98\linewidth,
            height=6.8cm,
            grid=major,
            grid style={dashed, gray!20}, 
            xlabel={Num. sampling steps (NFE)},
            ylabel={Generative Perplexity ($\downarrow$)},
            xmin=4, xmax=1024,   
            ymin=0, ymax=200,    
            ytick={0, 20, 40, 60, 80, 100, 120, 140, 160, 180, 200},
            xtick={8, 16, 32, 64, 128, 256, 512, 1024},
            xticklabels={8, 16, 32, 64, 128, 256, 512, 1024},
            xticklabel style={font=\footnotesize},
            ylabel style={font=\small, yshift=0.1cm},
            xlabel style={font=\small},
            legend pos=north east,
            legend style={
                nodes={scale=0.75, transform shape},
                fill=white,
                fill opacity=0.85,
                draw=gray!20,
                thin,
                cells={anchor=west}
            },
            cycle list name=mcd_sdtt_final
        ]
    
        \addplot coordinates {
            (8, 826.298) (16, 324.415) 
            (32, 163.2703) (64, 105.9561) (128, 81.544) (256, 65.016) (512, 54.742) 
            (1024, 43.7985)
        };
        \addlegendentry{Teacher (MDLM)}
    
        \addplot coordinates {
            (8, 343.352) (16, 156.436) (32, 94.972) (64, 67.288) 
            (128, 55.444) (256, 45.814) (512, 40.186) (1024, 31.634)
        };
        \addlegendentry{MCD (Round 1)}
    
        \addplot coordinates {
            (8, 308.400) (16, 134.491) (32, 79.742) (64, 56.996) 
            (128, 45.390) (256, 37.281) (512, 32.450) (1024, 26.656)
        };
        \addlegendentry{MCD (Round 2)}
    
        \addplot coordinates {
            (8, 236.803) (16, 103.144) (32, 61.821) (64, 45.030) 
            (128, 35.093) (256, 31.067) (512, 28.851) (1024, 23.009)
        };
        \addlegendentry{MCD (Round 3)}
    
        \addplot coordinates {
            (8, 164.311) (16, 80.874) (32, 46.545) (64, 35.017) 
            (128, 30.323) (256, 24.512) (512, 22.287) (1024, 18.682)
        };
        \addlegendentry{MCD (Round 4)}
    
        \addplot coordinates {
            (8, 118.098) (16, 61.351) (32, 39.727) (64, 29.735) 
            (128, 23.124) (256, 21.216) (512, 18.378) (1024, 17.361)
        };
        \addlegendentry{MCD (Round 5)}
    
        \end{semilogxaxis}
    \end{tikzpicture}
    \caption{Performance across Sampling Steps. Generative Perplexity (Gen PPL) for the Teacher model (MDLM, \textcolor{tea_blue}{\textbf{Grey-Blue}}) and successive rounds of Masked Consistency Distillation (MCD, \textcolor{mcd_r4}{\textbf{Orange}}). By zooming in on the PPL $\le 200$ regime, we highlight that MCD matches the teacher's performance with significant speedup.}
    \label{fig:gen_ppl}
\end{figure}

\begin{table}[t]
\centering
\caption{\textbf{Evolution of Generation Quality (Gen PPL $\downarrow$).} Comparison between Duo, SDTT, and our method (MCD) across 5 distillation rounds. Note that \textbf{SDTT and MCD share the same teacher (MDLM)}, while Duo utilizes a separate teacher checkpoint (Duo-T). We report the corresponding teacher baselines in the group headers. \textbf{Bold} indicates the best performance in each column.}
\label{tab:distillation_evolution}
\renewcommand{\arraystretch}{1.1}
\setlength{\tabcolsep}{3.2pt}
\begin{small}
\begin{sc}
\begin{tabular}{lccccc}
\toprule
& \multicolumn{5}{c}{\textbf{Distillation Round}} \\
\cmidrule(lr){2-6}
\textbf{Setting / Method} & \textbf{R1} & \textbf{R2} & \textbf{R3} & \textbf{R4} & \textbf{R5} \\
\midrule

\multicolumn{6}{l}{\textit{32 Steps} \quad \scriptsize Teacher: MDLM (SDTT/MCD) 163.27 \textbar\ Duo-T (Duo) 96.46} \\
\hspace{3mm} Duo & \textbf{82.03} & \textbf{71.64} & 66.37 & 62.53 & 60.98 \\
\hspace{3mm} SDTT & 119.35 & 90.74 & 76.03 & 63.88 & 52.16 \\
\hspace{3mm} MCD (Ours) & 94.97 & 79.74 & \textbf{61.82} & \textbf{46.55} & \textbf{39.73} \\
\addlinespace[0.4em]

\multicolumn{6}{l}{\textit{64 Steps} \quad \scriptsize Teacher: MDLM (SDTT/MCD) 105.96 \textbar\ Duo-T (Duo) 85.64} \\
\hspace{3mm} Duo & 73.48 & 66.18 & 61.27 & 56.64 & 55.75 \\
\hspace{3mm} SDTT & 81.64 & 65.89 & 52.17 & 47.40 & 40.56 \\
\hspace{3mm} MCD (Ours) & \textbf{67.79} & \textbf{57.00} & \textbf{45.03} & \textbf{35.02} & \textbf{29.74} \\
\addlinespace[0.4em]

\multicolumn{6}{l}{\textit{128 Steps} \quad \scriptsize Teacher: MDLM (SDTT/MCD) 81.54 \textbar\ Duo-T (Duo) 79.86} \\
\hspace{3mm} Duo & 71.05 & 62.45 & 58.67 & 54.96 & 54.14 \\
\hspace{3mm} SDTT & 65.58 & 53.76 & 42.41 & 38.61 & 32.93 \\
\hspace{3mm} MCD (Ours) & \textbf{55.44} & \textbf{45.39} & \textbf{35.09} & \textbf{30.32} & \textbf{23.12} \\
\addlinespace[0.4em]

\multicolumn{6}{l}{\textit{256 Steps} \quad \scriptsize Teacher: MDLM (SDTT/MCD) 65.02 \textbar\ Duo-T (Duo) 78.53} \\
\hspace{3mm} Duo & 68.36 & 59.25 & 56.67 & 54.40 & 53.69 \\
\hspace{3mm} SDTT & 53.09 & 43.08 & 35.12 & 33.32 & 27.31 \\
\hspace{3mm} MCD (Ours) & \textbf{45.81} & \textbf{37.29} & \textbf{31.07} & \textbf{24.51} & \textbf{21.22} \\
\addlinespace[0.4em]

\multicolumn{6}{l}{\textit{512 Steps} \quad \scriptsize Teacher: MDLM (SDTT/MCD) 54.74 \textbar\ Duo-T (Duo) 78.16} \\
\hspace{3mm} Duo & 67.94 & 61.07 & 56.27 & 53.54 & 52.15 \\
\hspace{3mm} SDTT & 43.57 & 37.72 & 32.94 & 30.98 & 25.60 \\
\hspace{3mm} MCD (Ours) & \textbf{40.19} & \textbf{32.45} & \textbf{28.85} & \textbf{22.29} & \textbf{18.38} \\

\bottomrule
\end{tabular}
\end{sc}
\end{small}
\end{table}

\subsection{Sample Quality and Efficiency}

We evaluate generation quality using Generative Perplexity (Gen PPL), computed by a pre-trained GPT-2 Large~\cite{radford2019language} model on the generated samples. Lower PPL indicates better sample quality.

\textbf{Superior Performance in Few-Step Generation.} As shown in Figure~\ref{fig:gen_ppl} and Table~\ref{tab:distillation_evolution}, MCD consistently outperforms the SDTT baseline across all sampling steps ($N \in \{32, \dots, 512\}$). 
Consistent with our premise, the distilled results confirm that masked diffusion maintains a significant performance advantage over the uniform diffusion framework employed by Duo~\cite{sahoo2025diffusion}.
Notably, at 32 sampling steps, MCD achieves a Gen PPL of 39.73, which not only surpasses the SDTT baseline (52.16) but also outperforms the original Teacher model sampled with 512 steps (54.74). This shows that MCD accelerates the inference by \textbf{16$\times$} while simultaneously improving the quality of the teacher's generation.

\textbf{Robustness at Extremely Low NFEs.} Generating coherent text at extremely low step counts ($N=8$) is notoriously difficult for discrete diffusion models due to the rapid accumulation of errors. MCD overcomes this limitation and maintains structural coherence. As shown in Figure~\ref{fig:gen_ppl}, MCD achieves a Gen PPL of 118.1 even at 8 steps. This robustness stems from our duality framework: by identifying the underlying latent Gaussian process, MCD recovers a deterministic trajectory that places both the student state $z_t$ and the teacher state $z_s$ on the same consistent path, ensuring coherence during few-step generation.

\textbf{Convergence Efficiency.} The results also highlight the superior convergence of our method. By distillation Round 4, MCD already surpasses the final performance of SDTT (Round 5) at most step counts. This efficiency is rooted in the theoretical framework of Masked Diffusion Duality. Unlike SDTT, which assumes no deterministic mapping exists for discrete diffusion and thus distills against high-variance stochastic targets~\cite{deschenaux2024beyond}, MCD identifies the underlying deterministic latent trajectory (Section~\ref{sec:locking}). This provides a stable, consistent regression target for the student, significantly accelerating training convergence compared to distilling stochastic jumps.

\subsection{Ablation Study}

We conduct an ablation study to validate the effectiveness of the loss function design in our framework. We report the results after the first round of distillation, evaluating Generative Perplexity (Gen PPL) at 32 sampling steps on OpenWebText. Results are summarized in Table~\ref{tab:ablation}.

\textbf{Loss Function Analysis.} We compare our proposed Hybrid Consistency Objective against two standard distillation objectives: Forward KL Divergence (KL-Fwd) and Backward KL Divergence (KL-Bwd).
First, regarding KL-Bwd, we found that using the backward KL divergence (mode-seeking) within our proposed framework led to severe optimization instability. As noted in Table~\ref{tab:ablation}, the loss values exploded, causing the distillation process to fail completely.
Second, for KL-Fwd, while optimization was stable, it resulted in suboptimal generation quality (PPL 110.86) compared to our method.
In contrast, the Hybrid Objective (Ours) combines the distributional supervision of KL divergence with the hard reconstruction signal of Cross-Entropy ($\mathcal{L}_{\text{CE}}$). This approach achieves the best performance (PPL 94.97), effectively integrating the teacher's soft probability estimates with hard ground-truth supervision.



\subsection{Scalability Analysis}
\label{sec:scalability}

To investigate the scalability of our approach, we applied MCD to a larger backbone, scaling the model size from 169M to 863M parameters (DiT-Large). 
Following the experimental protocol established by SDTT~\cite{deschenaux2024beyond}, we train the 863M parameter teacher model for 400K steps and using Llama3 8B~\cite{touvron2023llama} for evaluating. This ensures a strictly fair comparison with the strongest baselines under identical training budgets. 

As shown in Table~\ref{tab:scalability}, MCD remains effective in this large-scale setting. 
After just one round of distillation, MCD outperforms the SDTT baseline. For instance, at 64 sampling steps, MCD achieves a perplexity of 121.01 compared to 136.71 for SDTT, representing an 11.5\% relative improvement. 
Furthermore, the student model significantly outperforms the teacher across all sampling steps. 
Notably, at 32 steps, MCD reduces the perplexity from 260.17 to 204.10, and at 512 steps, it improves from 62.92 to 48.54. 
These results confirm that MCD scales effectively to larger architectures and consistently surpasses the teacher's generation quality.

\begin{table}[t]
\centering
\caption{Ablation of Loss Function. We report the results after the first round of distillation, evaluating Generative Perplexity (Gen PPL) at 32 sampling steps on OpenWebText. ``FAILED'' indicates training failure due to exploding loss values.}
\label{tab:ablation}
\begin{small}
\begin{tabular}{lc}
\toprule
\textbf{Setting} & \textbf{Gen PPL (32 steps)} \\
\midrule
\multicolumn{2}{l}{\textit{Loss Function Choices}} \\
\textbf{Hybrid Objective (Ours)} & \textbf{94.97} \\
KL-Fwd Objective & 110.86 \\
KL-Bwd Objective & \textsc{Failed} \\ 
\bottomrule
\vspace{-15pt}
\end{tabular}
\end{small}
\end{table}

\begin{table}[t]
\centering
\setlength{\tabcolsep}{4.0pt}
\caption{Comparison of distillation efficiency at Round 1 on DiT-Large (863M). Note that MCD achieves significant perplexity reduction ($\sim$21\%) immediately in the first round(R1). The absolute PPL values reflect the use of Llama 3 8B as the evaluator on the standard 400k-step teacher checkpoint, consistent with the protocol in SDTT.}
\label{tab:scalability}
\begin{small}
\begin{tabular}{lcccc}
\toprule
\textbf{Steps} & \textbf{Teacher} & \textbf{SDTT (R1)} & \textbf{MCD (R1)} & \textbf{Imp.} \\
\midrule
32  & 260.17 & 228.06 & \textbf{204.10} & 21.55\% \\
64  & 155.23 & 136.71 & \textbf{121.01} & 22.04\% \\
128 & 100.49 & 88.86  & \textbf{77.31}  & 23.07\% \\
256 & 75.68  & 67.43  & \textbf{58.06}  & 23.28\% \\
512 & 62.92  & 55.57  & \textbf{48.54}  & 22.85\% \\
\bottomrule
\vspace{-15pt}
\end{tabular}
\end{small}
\end{table}

\section{Related Work}
\label{sec:related_work}

\textbf{Diffusion Language Models}
Discrete diffusion models~\cite{austin2021structured} operate directly on categorical data, avoiding the quantization errors of continuous embeddings. They are broadly categorized into two categories, \textit{uniform} diffusion\cite{lou2023discrete} and \textit{masked} diffusion (MDLM)~\cite{sahoo2024simple, shi2024simplified, ou2024your}.
In contrast, Gaussian diffusion~\cite{li2022diffusion, han2023ssd, gulrajani2023likelihood} is used for language modeling by injecting noise into the continuous embeddings of discrete tokens.

\textbf{Acceleration and Distillation.} 
Distillation techniques in Gaussian diffusion models~\cite{luhman2021knowledge, salimans2022progressive, song2023consistency} rely on deterministic PF-ODE trajectories, which are unavailable for discrete diffusion.
However, recent work has begun to explore acceleration methods for discrete diffusion models.
SDTT~\cite{deschenaux2024beyond} pioneered self-distillation for masked models achieving significant acceleration, but relies on stochastic approximations assuming no deterministic mapping exists.
Duo~\cite{sahoo2025diffusion} successfully bridges the gap for \textit{uniform} diffusion, enabling consistency distillation, yet remains limited to the uniform noise schedule which typically underperforms masked modeling.
Concurrent to our work, CD\textsuperscript{4}LM~\cite{liang2026cd4lm} employs a similar subset masking strategy and demonstrates strong scalability through large-scale experiments. While their empirical success serves as a strong validation for the effectiveness of this coupling approach, they justify it primarily via Rao-Blackwellization for variance reduction.
In contrast, we prove that the discrete process is a deterministic projection of a latent Gaussian trajectory, providing the rigorous theoretical foundation that explains \textit{why} such scalar trajectory locking works strictly.





\section{Conclusion}

In this work, we first established the Masked Diffusion Duality, proving that the seemingly stochastic masking process arises as a deterministic projection of a latent Gaussian flow. 
This theoretical breakthrough led to the discovery of Scalar Trajectory Locking, which simplifies the complex high-dimensional latent coupling into a tractable scalar thresholding operation. 
Building on this foundation, we introduced Masked Consistency Distillation, a principled framework that leverages these insights to analytically construct deterministic trajectory pairs.
Empirical results demonstrate that MCD significantly outperforms state-of-the-art baselines, including SDTT and Duo, achieving a 16$\times$ speedup while surpassing the teacher's generation quality. Furthermore, our approach exhibits strong scalability to large-scale backbones. 
We hope this work provides a solid theoretical foundation for future research.

\nocite{langley00}

\bibliographystyle{icml2026}

\bibliography{example_paper}




\end{document}

%% file: fig_duality.tex
\begin{figure}[t]
\centering
\definecolor{mainBlue}{RGB}{0, 114, 178}
\definecolor{mainOrange}{RGB}{213, 94, 0}
\definecolor{textGray}{RGB}{80, 80, 80}

\resizebox{\columnwidth}{!}{%
\begin{tikzpicture}[
    font=\sffamily,
    >=stealth,
    thick,
    data_node/.style={
        rectangle, rounded corners=3pt,
        minimum width=1.6cm, minimum height=1.0cm,
        draw=black!80, line width=1.2pt,
        font=\Large\bfseries, fill=white,
        inner sep=3pt
    },
    path_label/.style={
        font=\small\bfseries, text=textGray, align=center
    },
    proj_box/.style={
        rectangle, rounded corners=4pt,
        draw=black!15, fill=white, line width=0.8pt,
        align=left, inner sep=5pt,
        drop shadow={opacity=0.08, shadow xshift=1pt, shadow yshift=-1pt}
    }
]


\node[data_node] (x0) at (0, 0) {$\mathbf{x}_0$};
\node[below=0.15cm of x0, font=\small, color=textGray] {Clean Data};

\node[data_node, draw=mainOrange, text=mainOrange, dashed] (wt) at (6.0, 0) {$\mathbf{w}_t$};
\node[below=0.15cm of wt, font=\small, color=mainOrange] {Latent Gaussian};

\node[data_node, draw=mainBlue, text=mainBlue, fill=mainBlue!5] (zt) at (6.0, 4.2) {$\mathbf{z}_t$};
\node[right=0.15cm of zt, align=left, font=\small, color=mainBlue] {Masked State\\(Observation)};


\draw[->, mainBlue, line width=1.5pt] (x0.north) to[out=85, in=180] 
    node[pos=0.35, path_label, text=mainBlue, above=18pt] {Discrete Masking\\$q(\mathbf{z}_t|\mathbf{x}_0)$} 
    (zt.west);

\draw[->, mainOrange, dashed, line width=1.5pt] (x0.east) -- 
    node[midway, path_label, text=mainOrange, below=5pt] {Gaussian Diffusion\\$q(\mathbf{w}_t|\mathbf{x}_0)$} 
    (wt.west);

\draw[->, black!80, line width=1.5pt] (wt.north) -- (zt.south);


\node[proj_box] at ($(wt)!0.5!(zt)$) (proj_module) {
    \textbf{Projection} $\mathcal{P}(\mathbf{w}_t)$\\[4pt]
    \scriptsize 
    \renewcommand{\arraystretch}{1.5} 
    ${\mathbf{z}_t} = \begin{cases} 
        \mathbf{x}_0 & \text{if } {\mathbf{w}_t^{(k)}} > \max\limits_{j \neq k} {\mathbf{w}_t^{(j)}} \\
        \mathbf{e}_K^{\texttt{[M]}} & \text{otherwise}
    \end{cases}$
};

\node[font=\bfseries, color=black!40, align=center] at (2.8, 2.0) {Structural\\Duality\\[2pt]\Huge $\equiv$};

\end{tikzpicture}
}
\caption{An illustration of the Masked Diffusion Duality. 
Top: The discrete masking process (blue) directly samples $\mathbf{z}_t$ from $\mathbf{x}_0$. 
Bottom: The underlying continuous latent Gaussian process (orange). 
The projection operator $\mathcal{P}$ maps the continuous latents $\mathbf{w}_t \in \mathbb{R}^K$ to the discrete observations $\mathbf{z}_t$ by thresholding the signal strength against the maximum relative noise. This establishes a deterministic link between the marginal distributions of the two~processes.}
\label{fig:duality_eq5}
\end{figure}